%% file: Template.tex
\documentclass{article}
\usepackage{subfig}
\usepackage[labelfont=bf, justification=justified]{caption}

\usepackage{spconf,amsmath,graphicx}
\usepackage{amsmath}
\usepackage{amsthm}
\usepackage{amsfonts}
\usepackage{bbm}
\usepackage{ifpdf}
\usepackage{float}
\usepackage{fancyhdr}
\usepackage{xcolor}
\usepackage{multirow}
\usepackage{bm}
\usepackage[toc,page]{appendix}
\usepackage{hyperref}

\title{Fast graph kernel with optical random features}
\name{ \qquad Hashem Ghanem\qquad Nicolas Keriven \qquad Nicolas Tremblay\thanks{The authors warmly thank LightOn (\url{https://lighton.ai}) for the use of the OPU. HG is supported by the LabEx PERSYVAL (ANR-11-LABX-0025-01) \emph{``Bourse d'excellence''}. NT is partly supported by the French National Research Agency in the framework of the "Investissements d’avenir” program (ANR-15-IDEX-02) and the LabEx PERSYVAL (ANR-11-LABX-0025-01); as well as  the ANR GraVa (ANR-18-CE40-0005).} }

\address{ CNRS, GIPSA-lab, FR-38402 Saint Martin d’Heres Cedex, France}
\begin{document}
%
\newtheorem{theorem}{Theorem} 
\maketitle
\begin{abstract}
The graphlet kernel is a classical method in graph classification. It however suffers from a high computation cost due to the isomorphism test it includes. As a generic proxy, and in general at the cost of losing some information, this test can be efficiently replaced by a user-defined mapping that computes various graph characteristics. In this paper, we propose to leverage \emph{kernel random features} within the graphlet framework, and establish a theoretical link with a mean kernel metric. If this method can still be prohibitively costly for usual random features, we then incorporate \emph{optical} random features that can be computed in \emph{constant time}. Experiments show that the resulting algorithm is orders of magnitude faster that the graphlet kernel for the same, or better, accuracy.
\end{abstract}
\section{Introduction}
\label{sec:intro}
In mathematics and data science, graphs are used to model a set of objects (the \emph{nodes}) and their interactions (the \emph{edges}). 
Given a set of pre-labeled graphs $(\mathcal{X}=\{\G_1,\ldots,\G_n\}, \mathcal{Y}=\{y_1,\ldots,y_n\})$, where each graph $\G_i$ belongs to the class with label $y_i$, graph classification consists in designing an algorithm that outputs the class label of a new graph.
For instance, proteins can be modeled as graphs: amino acids are nodes and the chemical links between them are edges. They can be classified to enzymes and non-enzymes \cite{protein_application}.
In social networks analysis, post threads can be modeled with graphs whose nodes are users and edges are replies to others' comment \cite{graph_soc_net}. One task is then to discriminate between discussion-based and question/answer-based threads \cite{class_Reddit}.
%
In addition to the graph structure, nodes and edges may have extra features. While it has been shown that node features are important to obtain high classification performance \cite{node_features}, here we focus on the case where one has only access to the graph structure.

Structure-based graph classification has been tackled with many algorithms. Frequent subgraphs based algorithms \cite{frequent_subgraphs} analyze the graph dataset $\mathcal{X}$ to catch the frequent and discriminative subgraphs and use them as features. Kernel-based algorithms \cite{kriege_graph_kernels} can be used by defining similarity functions (kernels) between graphs. An early and popular example is the \emph{graphlet kernel}, which computes frequencies of subgraphs. It is however known to be quite costly to compute \cite{graphlet_kernel}, in particular due to the presence of graph isomorphism tests. While possible in particular cases \cite{graphlet_kernel}, accelerating the graphlet kernel for arbitrary datasets remains open.
Finally, graph neural networks (GNNs) \cite{Bronstein2017} have recently become very popular in graph machine learning. They are however known to exhibit limited performance when node features are unavailable \cite{GCN_powerful}. 

In kernel methods, random features are an efficient approximation method~\cite{rahimi2008random, RF_1}. Recently, it has been shown \cite{saade_opu} that \emph{optical computing} can be leveraged to compute such random features in \emph{constant time} in \emph{any dimension} -- within the limitations of the current hardware, here referred to as Optical Processing Units (OPUs).
The main goal of this paper is to provide a proof-of-concept answer to the following question: can OPU computations be used to reduce the computational complexity of a combinatorial problem like the graphlet kernel? Drawing on a connection with mean kernels and Maximum Mean Discrepancy (MMD) \cite{Gretton2007}, we show, empirically and theoretically, that a fast and efficient graph classifier can indeed be obtained with OPU computations.

\section{Background}
\label{sec:background}
First, we present the concepts necessary to define the graphlet kernel. We represent a graph of size $v$ by the adjacency matrix $\mathbf{A}\in \{0,1\}^{v\times v}$, such that $a_{i,j} =1$ if there is an edge between nodes $\{i,j\}$ and $0$ otherwise. Two graphs are said to be isomorphic ($\G\cong \G')$ if we can permute the nodes of one such that their adjacency matrices are equal \cite{isomorphism}. 

\subsection{Isomorphic graphlets}

In this paper, we will, depending on the context, manipulate two different notions of $k$-graphlets (that is, small graphs of size $k$): with or without discriminating isomorphic graphlets. We denote by $\bar{\mathfrak{H}}=\{\bar{\phlet}_1,..., \bar{\phlet}_{\bar N_k}\}$ with $\bar{N}_k = 2^{\frac{k(k-1)}{2}}$ the set of all size-$k$ graphs, where isomorphic graphs are counted multiple times, and $\mathfrak{H}=\{\phlet_1,..., \phlet_{N_k}\} \subset \bar{\mathfrak{H}}$ the set of all non-isomorphic graphs of size $k$. Its size $N_k$ has a (quite verbose) closed-from expression \cite{oeis}, but is still exponential in $k$.
In the course of this paper, we shall manipulate mappings $\varphi(\mathcal{H})$ and probability distributions (histograms) over graphlets. By default the underlying space will be $\bar{\mathfrak{H}}$, however when the mapping $\varphi$ is \emph{permutation-invariant}, then the underlying space can be thought of as $\mathfrak{H}$. 
Also note that, assuming each isomorphic copies has equal probability, a probability distribution over $\bar{\mathfrak{H}}$ can be \emph{folded} into one over $\mathfrak{H}$, and both distributions \emph{contain the same amount of information}.

\subsection{The graphlet kernel}

The traditional graphlet kernel is defined by computing histograms of subgraphs over non-isomorphic graphlets $\mathfrak{H}$. We define the matching function $\varphi_k^{match}(\mathcal{F}) = \left[ 1_{(\mathcal{F} \cong \phlet_i)}\right]_{i=1}^{N_k} \in \{0,1\}^{N_k}$, where $\mathcal{F}$ is a graph of size $k$. In words, $\match(\mathcal{F})$ is a one-hot vector of dimension $N_k$ identifying $\mathcal{F}$ up to isomorphism. Note that the cost of evaluating $\varphi_k^{match}$ once is $O\left(N_k C^{\cong}_k\right)$, where $C^{\cong}_k$ is the cost of the isomorphism test between two graphs of size $k$, for which no polynomial algorithm is known \cite{isomorphism_np}. 
Given a graph $\G$ of size $v$, let $\mathfrak{F}_\G=\{\mathcal{F}_1,\mathcal{F}_2,\ldots,\mathcal{F}_{\binom{v}{k}}\}$ be the collection of subgraphs induced by all size-$k$ subsets of nodes. The following representation vector is called the $k$-spectrum of $\G$:
\begin{equation}
\label{eq:gk}
\mathsmaller{\mathbf{f}_\G=\binom{v}{k}^{-1}\sum_{\mathcal{F}\in\mathfrak{F}_\mathcal{G}} \match (\mathcal{F}) \in \R^{N_k}}
\end{equation}
For two graphs $\G,\G'$, the graphlet kernel \cite{graphlet_kernel} is then defined as $\bld{f}_\G^T\bld{f}_{\G'}$. 
For a graph of size $v$, the computation cost of $\mathbf{f}_\G$ is $C_{gk}= \mathcal{O}\left(\tbinom{v}{k}N_k C^{\cong}_k\right)$. This cost is usually prohibitively expensive, since each three terms are exponential in $k$.

Subgraph sampling is generally used as a first step to accelerate (and sometimes modify) the graphlet kernel \cite{graphlet_kernel}. Given a graph $\G$, we denote by $S_k(\G)$ a sampling process that yields a random subgraph of $\G$, seen as a probability distribution over $\bar{\mathfrak{H}}$. 
Then, sampling $s$ subgraphs $\hat{\mathfrak{F}}_\G = \{F_1,...,F_s\}$ $i.i.d.$ from $S_k(\G)$, we define the estimator:
\begin{align}
	\label{eq:fhat_unif}
	\mathsmaller{\hat{\mathbf{f}}_{\mathcal{G},S_k} =s^{-1}\sum_{F\in\hat{\mathfrak{F}}_\G} \varphi^{match}_k(F).}
\end{align}
and its expectation $\mathbf{f}_{\mathcal{G},S} = \mathbb{E}_{F \sim S_k(\G)} ~\varphi^{match}_k(F)$, which is nothing more than the \emph{folding} of the distribution $S_k(\G)$ over non-isomorphic graphlets $\mathfrak{H}$. For any sampler, we refer to these expectations as graphlet kernels. In all generality, any choice of sampling procedure $S_k$ yields a different definition of graphlet kernel. For instance, if one considers uniform sampling ($S^{\rm unif}$: independently samples $k$ nodes of $\G$ without replacement), then one obtains the original graphlet kernel of Eq.~\eqref{eq:gk}: $\mathbf{f}_{\G, S^{\rm unif}} = \mathbf{f}_\G$. Other choices of sampling procedures are possible~\cite{leskovec2006sampling}. In this paper, we will also use the random walk (RW) sampler, which, unlike uniform sampling, tends to sample connected subgraphs. 
The computation cost per graph of the approximate graphlet kernel of Eq.~\eqref{eq:fhat_unif} is $C_{gk + gs}= \mathcal{O}\left(s C_S N_k C^{\cong}_k\right)$, where $C_S$ is the cost of sampling one subgraph. 
For a fixed error in estimating $\bld{f}_{\G,S}$, the required number of samples $s$ generally needs to be proportional to $N_k$ \cite{graphlet_kernel}, which unfortunately still yields a generally unaffordable algorithm.

\begin{algorithm}
	\DontPrintSemicolon
	\KwInput{labeled graph dataset $\mathcal{X}=(\G_i,y_i)_{i=1,\ldots,n}$}
	\tools{Graphlet sampler $S_k$, a function $\varphi:\bar{\mathfrak{H}}\to \mathbb{R}^m$, a linear classifier (ex. SVM) }\\
	\Hyp{$s$: number of graphlet samples}\\
	\KwOutput{Trained model to classify graphs}
	\Algo{\\}
	\For{$\G_i$ in $\mathcal{X}$}{
		$\mathbf{f}_i=\mathbf{0}$ (null vector of size $m$) \\
		\For{$j=1:s$}{
			$F_{i,j}\gets S_k(\G_i)$\\
			$\mathbf{f}_i\gets \mathbf{f}_i +\frac{1}{s}\varphi(F_{i,j})$
		}
	}
	Train the classifier on the vector dataset $(\mathbf{f}_i,y_i)_{i=1}^n$
	\caption{GSA-$\varphi$ generic algorithm}\label{alg:gsa}
\end{algorithm}

\section{Graphlet kernel with optical maps} \label{ssed to get a lowerec:pagestyle}
\subsection{Proposed method}
\label{sec:algo}
In this paper, we focus on the main remaining bottleneck of the graphlet kernel, that is, the function $\match$. We define a framework where it is replaced with a user-defined map $\varphi: \bar{\mathfrak{H}} \to \mathbb{R}^m$, which leads to the final representation:
\begin{align}
	\label{eq:fhat_phi}
	\mathsmaller{\hat{\mathbf{f}}_{\mathcal{G},S_k,\varphi} =s^{-1}\sum_{F\in\hat{\mathfrak{F}}_\G} \varphi(F).}
\end{align}
The resulting methodology (Alg.~\ref{alg:gsa}) is referred to as \emph{Graphlet Sampling and Averaging} (GSA-$\varphi$). The cost of computing \eqref{eq:fhat_phi} is $C_{GSA-\varphi}= \mathcal{O}\left(s C_S C_{\varphi}\right)$, where $C_{\varphi}$ is the cost of applying $\varphi$.
%
%
Similar methods have been studied with $\varphi$ as simple graph statistics \cite{Dutta2018}, which unavoidably incurs information loss. We see next that choosing $\varphi$ as \emph{kernel random features} preserves information for a sufficient number of features. Some of these maps will \emph{not} be permutation-invariant at the graphlet level, however, in the infinite sample limit, it is easy to see that $\hat{\mathbf{f}}_{\mathcal{G},S_k,\varphi}$ is permutation-invariant at the graph level.

\subsection{Kernel random features with $GSA-\varphi$} 
\label{sec:MMD}

In the graphlet kernel, the underlying metric used to compare graphs is the Euclidean distance between graphlet histograms. When $\match$ is replaced by another $\varphi$, one compares certain \emph{embeddings} of distributions, which is reminiscent of kernel mean embeddings \cite{Gretton2007}. We show below that this corresponds to choosing $\varphi$ as kernel random features.

For two objects $\mathbf{x}, \mathbf{x'}$, a kernel $\kappa$ associated to a random features (RF) decomposition is a positive definite function that can be decomposed  as follows \cite{rahimi2008random}:
\begin{equation}
\label{eq:RF_decomposition}
\kappa(\mathbf{x},\mathbf{x}')=\mathbb{E}_{\mathbf{w}\sim p}[ \xi_\mathbf{w}(\mathbf{x})^* \xi_\mathbf{w}(\mathbf{x}')]
\end{equation}
where $\xi$ is a real (or complex) function parameterized by $\mathbf{w} \in \Omega$, and $p$ a probability distribution on $\Omega$. A classical example is the Fourier decomposition of translation-invariant kernels~\cite{rahimi2008random}.
The RF methodology then defines maps:
\begin{equation}
	\label{eq:def_RF}
	\mathsmaller{
	\varphi(\mathbf{x}) = m^{-1/2} ( \xi_{\mathbf{w}_j}(\mathbf{x}) )_{j=1}^m \in \mathbb{C}^m}
\end{equation}
where $m$ is the number of features and the parameters $\mathbf{w}_j$ are drawn iid from $p$. Then, $\kappa(\mathbf{x},\mathbf{x}')\approx	\varphi(\mathbf{x})^H	\varphi(\mathbf{x}') = m^{-1} \sum_j \xi_{\mathbf{w}_j}(\mathbf{x})^* \xi_{\mathbf{w}_j}(\mathbf{x}')$. 

Assume that we have a base kernel $\kappa(\F,\F')$ between \emph{graphlets}, with a RF decomposition $(\xi_\mathbf{w}, p)$, and define $\varphi$ as in \eqref{eq:def_RF}. Then, one can show \cite{Keriven2017a, Keriven2018} that the Euclidean distance between the embeddings \eqref{eq:fhat_phi} approximates the following \emph{Maximum Mean Discrepancy} (MMD) \cite{Gretton2007, Sriperumbudur2010} between distributions on $\bar{\mathfrak{H}}$:
 \begin{multline}\label{eq:MMD}
 \mathsmaller{
 \textup{MMD}^2(S_k(\G), S_k(\G'))} \\
 \mathsmaller{= \mathbb{E}_{\mathbf{w}} \left( \left| \mathbb{E}_{S_k(\G)} \xi_\mathbf{w}(F) - \mathbb{E}_{S_k(\G')} \xi_\mathbf{w}(F') \right|^2 \right)}
 \end{multline}
The main property of the MMD is that, for so-called \emph{characteristic kernels}, it is a true metric on distributions, \emph{i.e.} $\textup{MMD}(S_k(\G), S_k(\G')) = 0 \Leftrightarrow S_k(\G) = S_k(\G')$. 
Most usual kernels, like the Gaussian kernel, are characteristic \cite{Gretton2007, Sriperumbudur2010}.
%


\begin{theorem}\label{theorem:concentration}
Let $\G$ and $\G'$ be two graphs. 
Assume a random feature map as in \eqref{eq:def_RF}. Assume that $|\xi_\mathbf{w}(F)| \leq 1$ for any $\mathbf{w},F$.
We have for all $\delta>0$ and with probability at least $1-\delta$:
\begin{multline}\label{eq:bound}
\mathsmaller{
 \Big|\|\hat{\mathbf{f}}_{\G,S_k,\varphi} - \hat{\mathbf{f}}_{\G',S_k,\varphi} \|_2^2 - \textup{MMD}^2(S_k(\G),S_k(\G')) \Big| } \\
 \mathsmaller{ \leq 4 m^{-\frac12} \sqrt{\log (6/\delta)} + 8s^{-\frac12} \left(1+\sqrt{2\log(3/\delta)}\right)}
\end{multline}
\end{theorem}
\begin{proof}
	See the supplementary material~\cite{arxiv_version}.
\end{proof}
Hence, if two classes of graphs are well-separated in terms of the MMD \eqref{eq:MMD}, then, for sufficiently large $m,s$, GSA-$\varphi$ has the same classification power. However, according to \eqref{eq:bound}, $m$ should be of the order of $s$, and we have seen that the latter generally needs to be quite large:  most usual random feature scheme, typically in $C_\varphi =O(k^2 m)$, still have a high computation cost. We discuss next the use of \emph{optical hardware}.

\subsection{Considered choices of $\varphi_{RF}$}
\label{sec:phi_choices}
\textbf{Gaussian maps}: the RF map of the Gaussian kernel \cite{rahimi2008random}. 
\begin{equation}
\label{eq:Gaussian_map}
 \mathsmaller{\varphi_{Gs}(\F) = m^{-1/2} \left( \sqrt{2} \cos(\mathbf{w}_j^T\mathbf{a}_\F+b_j) \right)_{j=1}^m \in \mathbb{R}^m}
\end{equation}
where $\mathbf{a}_\F=flatten(\mathbf{A}_\F)$ is the vectorized adjacency matrix of the graphlet $\F$, 
the $\mathbf{w}_j\in \R^{k^2}$ are drawn from a Gaussian distribution and $b_j \sim \mathcal{U}([0, 2\pi])$. While using a Gaussian kernel on $\mathbf{a}_\mathcal{F}$ is not very intuitive, this will serve as a baseline for other methods. Note that $\varphi_{Gs}$ is not permutation-invariant.

\noindent\textbf{Gaussian maps applied on the sorted eigenvalues:} We consider a permutation-invariant alternative to the first case. 
For a graphlet $\F$ we denote the vector of its \emph{sorted} eigenvalues by $\bm{\lambda}(\F) \in\R^k$ and $\varphi_{Gs+eig}(\F) = \varphi_{Gs}(\bm{\lambda}(\F))$ (with $\mathbf{w}_j$ of dimension $k$). Note that the existence of co-spectral graphs, that is, non-isomorphic graphs with the same set of eigenvalues, implies a loss of information when computing $\bm{\lambda}(\F)$.

\noindent\textbf{Optical random feature maps:} Due to high-dimensional matrix multiplication, Gaussian RFs cost $\mathcal{O}(mk^2)$ and are notably expensive to compute in high-dimension (here mostly large $m$). To solve this, OPUs (Optical Processing Units) were recently developed to compute a specific random features mapping in \emph{constant time  $\mathcal{O}(1)$} using light scattering~\cite{saade_opu} -- within the physical limits of the OPU, currently of the order of a few millions for both input and output dimensions. Here we again consider the flattened adjacency matrix for simplicity. The OPU computes an operation of the type:
\[
\label{OPU_equation}
\mathbf{\varphi}_{OPU}(\F)= m^{-1/2}\left(|\mathbf{w}_j^T \mathbf{a}_\F+\mathbf{b}_j|^2\right)_{j=1}^m
\]
with $\mathbf{b}_j$ a random bias and $\mathbf{w}_j$ a complex vector with Gaussian real and imaginary parts. Both $\mathbf{w}_j, \mathbf{b}_j$ are here incurred by the physics and are unknown, however the corresponding kernel $\kappa(\F,\F')$ has a closed-form expression~\cite{saade_opu}.
Table \ref{tab:cost} summarizes the complexities of the mappings $\varphi$ examined.

\begin{table}
\centering
\begin{tabular}{|c|c|c|}
\hline
\multicolumn{2}{|c|}{Graphlet kernel} & $O(\tbinom{v}{k} N_k C^{\cong}_k)$\\ \hline \hline
\multirow{4}{*}{GSA-$\varphi$ with:} & $\varphi^{match}_k$ & $O(C_S s N_k C^{\cong}_k)$ \\
& $\varphi_{Gs}$ & $O(C_S s m k^2)$ \\ 
& $\varphi_{Gs+eig}$  & $O(C_S s (m k + k^3))$ \\ 
& $\varphi_{OPU}$  & $O(C_S s)$ \\ \hline
\end{tabular}
\caption{Per-graph complexities of GSA-$\varphi$.}
\label{tab:cost}
\end{table}

\section{Experiments}\label{sec:experiments}
\subsection{Datasets}\label{sec:setup}
%
Different methods are first compared in a controlled setting: a synthetic dataset generated by a \emph{Stochastic Block Model (SBM)} \cite{SBM}. We generate $300$ graphs, $240$ for training  and $60$ for testing. Each graph has $v=60$ nodes divided equally in six communities. Graphs are divided into two classes $\{0 , 1\}$. For each class we fix $p_{in}$ (resp. $p_{out}$) the edge probability between any two nodes in the same (resp. different) community. Also, to prevent the classes from being easily discriminated by the average degree, the pairs $(p_{in,i} , p_{out,i})_{i=0,1}$ are chosen such that nodes in both classes have the same expected degree (set to $10$). Having one degree of freedom left, we fix $p_{in,1}$ to $0.3$, and vary $r=(p_{in,1}/p_{in,0})$ the inter-class similarity parameter: the closer $r$ is to $1$, the more similar both classes are and the harder it is to discriminate them.

In addition, two real-world datasets are considered: D\&D \cite{DD_ref} (of size $n=1178$) and Reddit-Binary \cite{class_Reddit} ($n=2000$). We recall that graphs are classified based on their structure only and all other existing information is discarded. 
Python codes can be found here:  \url{https://github.com/hashemghanem/OPU_Graph_Classifier}. The OPU is developped by LightOn (\url{https://lighton.ai}). The rest of the experiments are performed on a laptop.



\begin{figure}
\centering
	%
	\includegraphics[width=4.4cm]{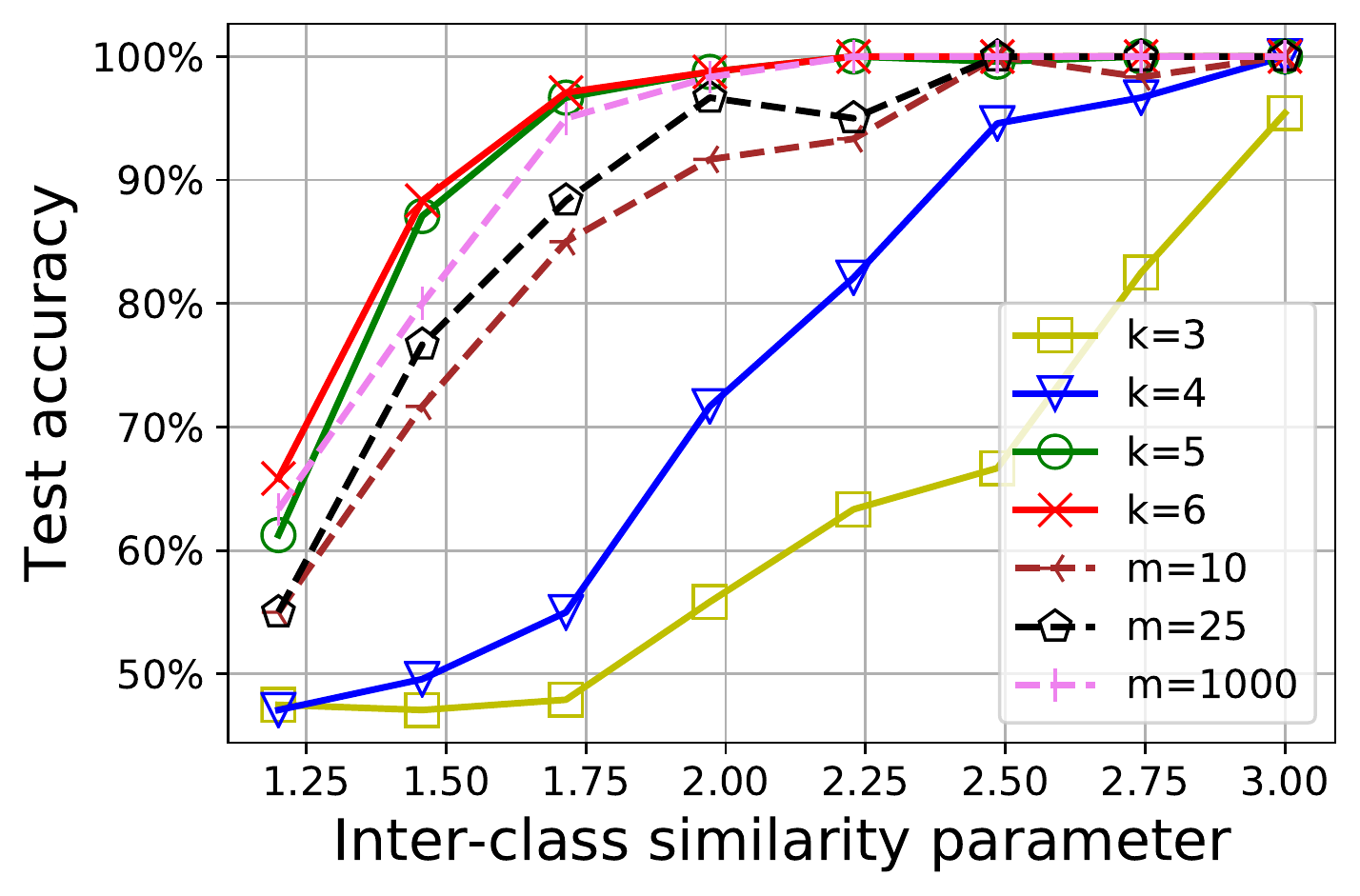}
%
		\includegraphics[width=4.4cm]{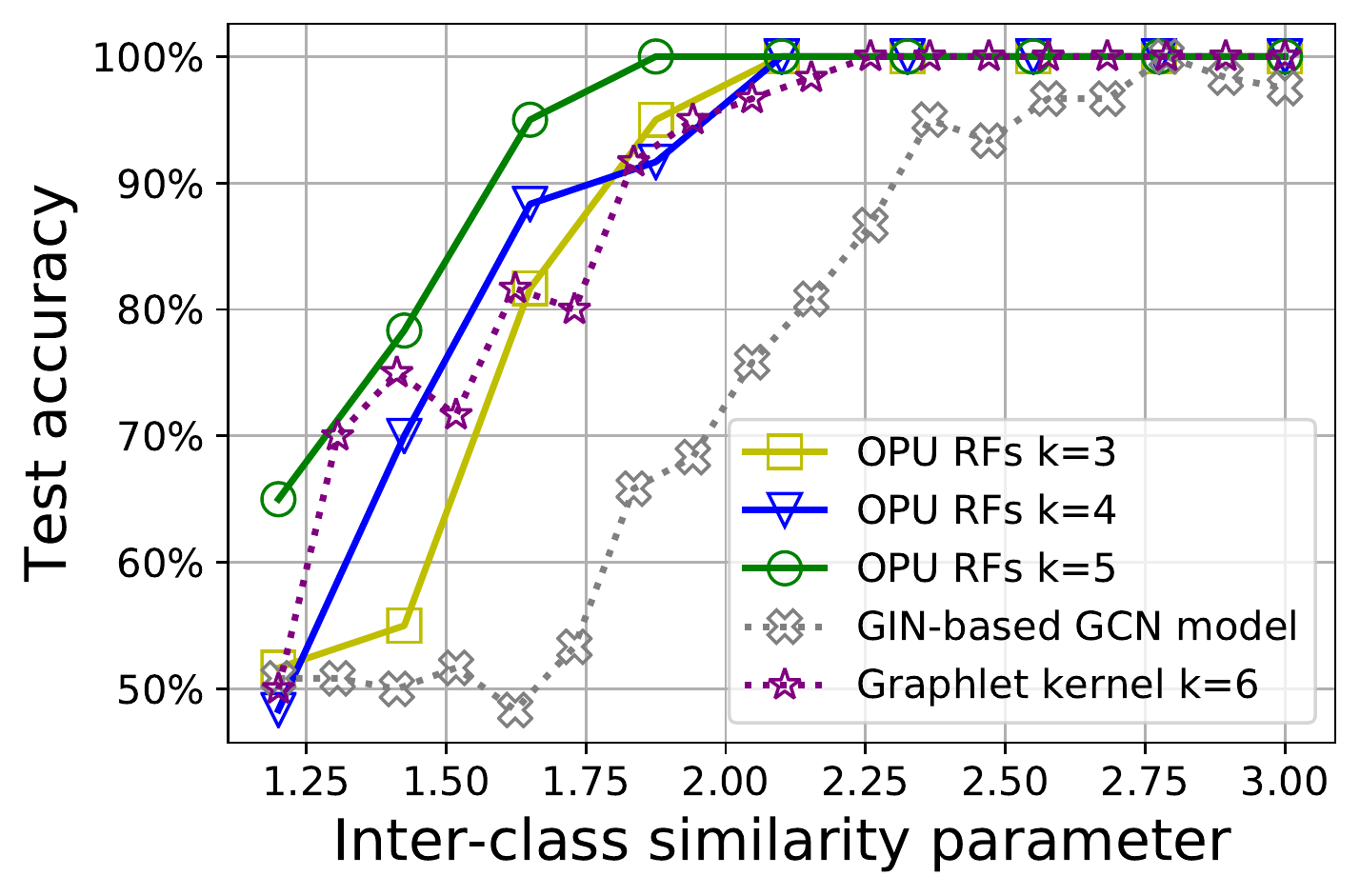}
	%
	\caption{(left)~$GSA-\varphi_{OPU}$ with uniform sampling ($s=2000$) for different $k$ ($m$ fixed to $5000$) and $m$ ($k$ fixed to $6$). (right)~Comparison between i/~$GSA-\varphi_{OPU}$ with RW sampling for different $k$; ii/~$GSA-\match$ with same number of samples and random features ($s=2000$, $m=5000$); iii/~the GNN consisting of 5 GIN layers \cite{GCN_powerful} followed by 2 fully connected layers, with dimension of hidden layers 4.
}
	\label{fig:GCN}
\end{figure}

\subsection{Varying $m, k$ and $S_k$ in $GSA-{\varphi_{OPU}}$}
From Fig. \ref{fig:GCN} (left), we observe that as $k$ and/or $m$ increase, the performance of $GSA-{\varphi_{OPU}}$ associated to uniform sampling increases, saturating in this SBM dataset for $m=5000$ and $k=6$. From the right figure, as expected, note that RW sampling outperforms uniform sampling: the smaller $k$, the larger the improvement.

\subsection{Choice of feature map $\varphi$}
\textbf{Comparison of random features}. Fig \ref{fig:diff_phi} (left) shows  that, for sufficiently large $m$, $GSA-\varphi_{OPU}$  outperforms both $GSA-\varphi_{Gs+Eig}$  and $GSA-\varphi_{Gs}$ (whose variance $\sigma^2$ is chosen so as to maximize the validation accuracy). 

\noindent\textbf{Comparing $GSA-\varphi_{OPU}$ and $GSA-\match$}. From Fig~\ref{fig:GCN} (right) we observe that with $s=2000$ and $m=5000$, $GSA-\varphi_{OPU}$ with both uniform sampling ($k=6$) and RW sampling ($k=5$)  clearly outperforms the graphlet kernel 
with $k=6$. 

\noindent\textbf{Computational time}. Fig \ref{fig:diff_phi} (right) compares computation time per subgraph, with respect to the subgraph size $k$. 
As expected, the execution time is exponential with $k$ for $GSA-\match$, roughly polynomial for $GSA-\varphi_{Gs}$ and $GSA-\varphi_{Gs+Eig}$, and constant for $GSA-\varphi_{OPU}$.


\subsection{Comparing $GSA-\varphi_{OPU}$ and a GNN model}\label{sec:vs_GIN}

In Fig \ref{fig:GCN}, we see that $GSA-\varphi_{OPU}$ with either RW sampling  ($k\geq4$) or uniform sampling ($k\geq5$) performs better than a deep GNN, specifically a model based on GIN layers~\cite{GCN_powerful}. GNNs are known to struggle in the absence of node features. 

\begin{figure}
\centering
	%
		\includegraphics[width=4.4cm]{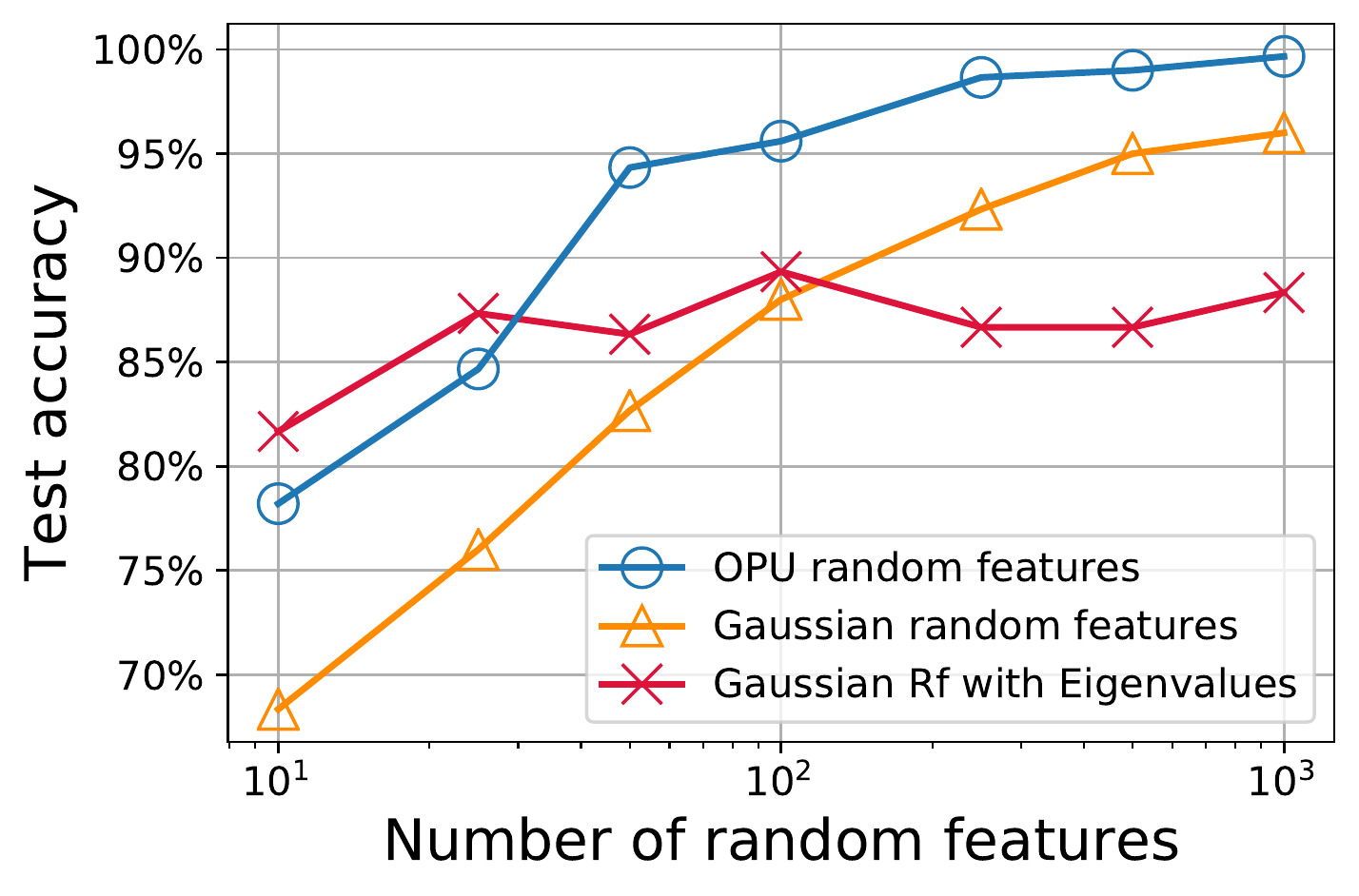}
		\includegraphics[width=4.4cm]{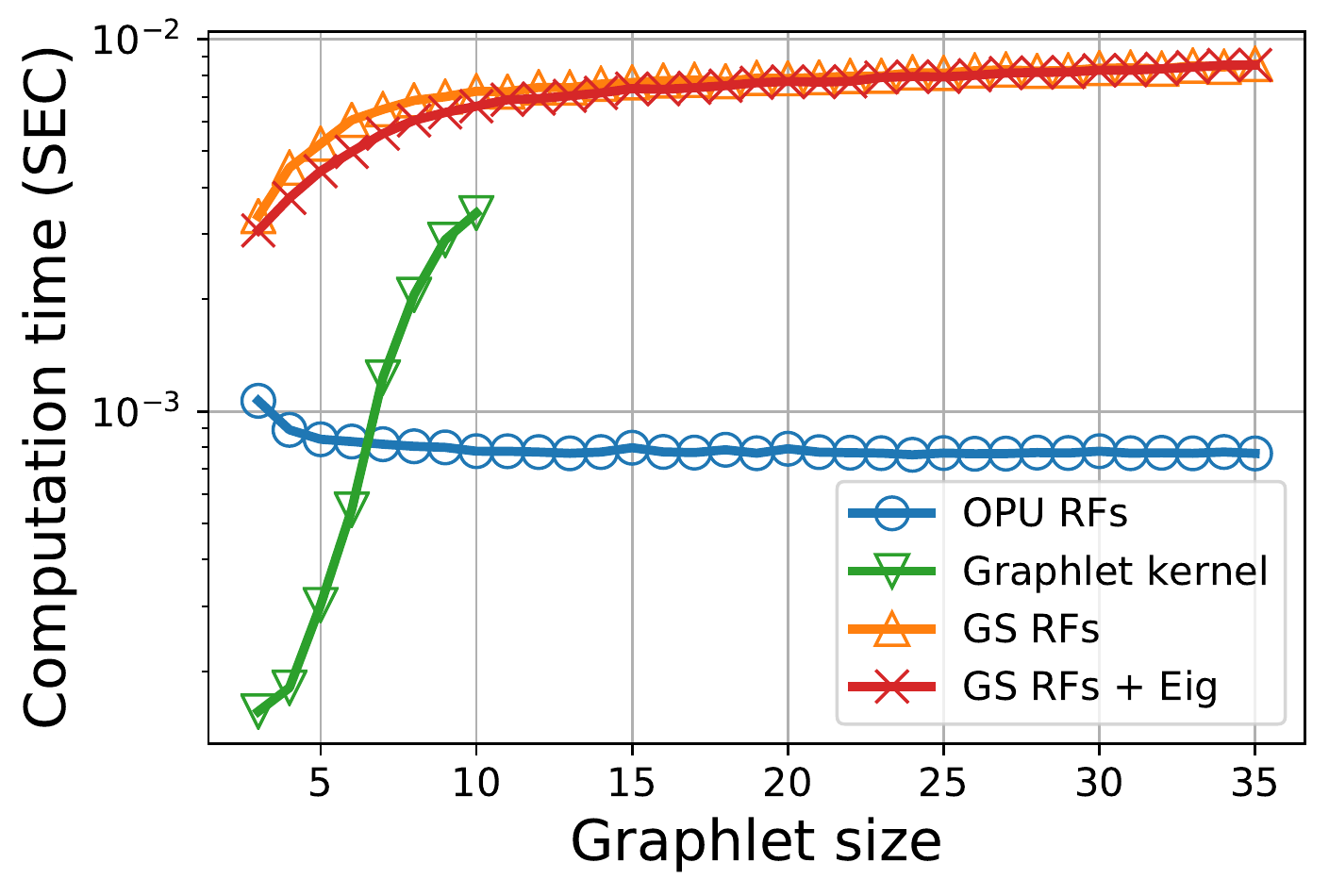}
	%
	\caption{ (left) Test accuracy versus $m$, for different maps  $\varphi$ in $GSA-\varphi$. (right) Computation time versus $k$ for $GSA-\varphi$ and the graphlet kernel. These figures are for $r=1.1$, $s=2000$, $m=5000$ and a Gaussian map variance $\sigma^2=0.01$.}
	\label{fig:diff_phi}
\end{figure}

\begin{figure}
\centering
%
  \includegraphics[width=4.4cm]{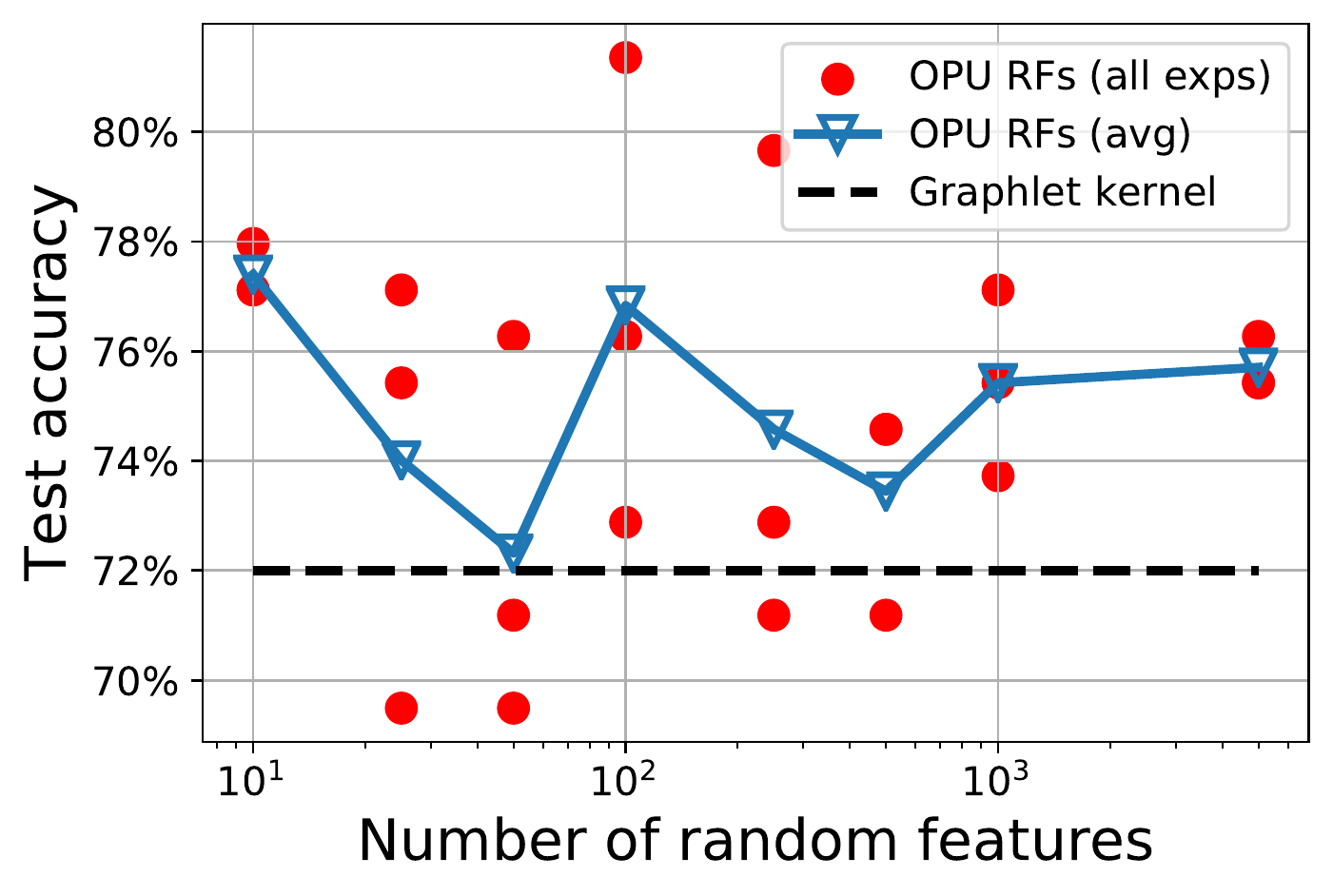}
  \includegraphics[width=4.4cm]{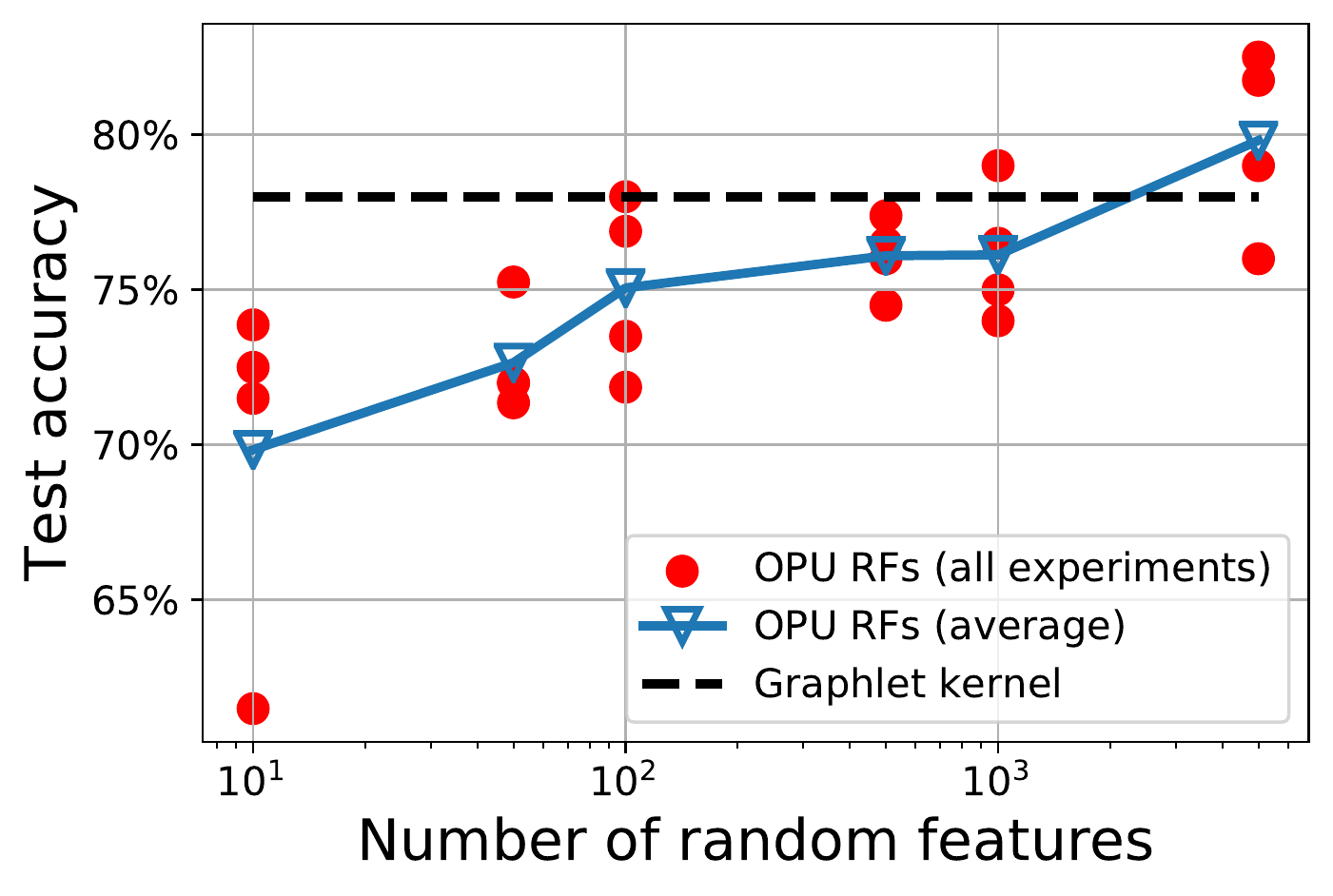}
%
\caption{ $GSA-\varphi$ vs the graphlet kernel on real datasets. (left)~D\&D. (right) Reddit-Binary. With $s=4000$, and $k=7$.}
\label{fig:DD}
\end{figure}

\subsection{$GSA-\varphi_{OPU}$  on real datasets}\label{sec:DD_Reddit}
Fig \ref{fig:DD} shows the test accuracy of $GSA-\varphi_{OPU}$ versus $m$, for two real datasets, compared with the graphlet kernel as a baseline. For each value of $m$ we conduct the experiment 3 times on D\&D and 4 times on Reddit-Binary dataset. For  D\&D, no steady improvement in the average accuracy is observed, however its variance between experiments decreases as $m$ increases. 
This average is still better than the accuracy obtained by $GSA-\match$. For Reddit-Binary, 
the average accuracy increases with the number of random features, and outperforms $GSA-\match$ for $m\geq5000$. 


\section{Conclusion}
\label{sec:Conclusion}

In this paper, we deployed OPUs random features in graph classification within the graphlet kernel framework. On the theoretical side, we showed concentration of the random embedding around an MMD metric between graphlet histograms, hinting at the potential of OPUs in such a combinatorial setting to reduce computation costs without information loss. On the practical side, our experiments showed that the proposed algorithm is significantly faster than the graphlet kernel with generally better performance. 
A major point left open is how to use the OPU in the presence of node features. A promising possibility is to integrate it within a message-passing framework to efficienty generate high-level embeddings that also use node features.
%
On the theoretical side, the properties of the MMD metric could be further analyzed on particular models of graphs such as SBMs.

\bibliographystyle{IEEEbib}

\bibliography{strings,refs}
\normalsize
\include{app}
\end{document}

%% file: app.tex
\begin{appendices}
\section{Proof of Theorem 1}\label{app:proof}
\begin{proof} We decompose the proof in two steps.

\textbf{Step 1: infinite $s$, finite $m$.} First we define the random variables $x_j=| \mathbb{E}_{F \sim S_k(\G)} \xi_{\mathbf{w}_j}(F) - \mathbb{E}_{F' \sim S_k(\G')} \xi_{\mathbf{w}_j}(F') |^2$, which are: i/independent, ii/have expectation $MMD(S_k(\G),S_k(\G'))^2$, /iii are bounded by the interval $[0,4]$ based on our assumption $|\xi_w|\leq 1$. Thus, as a straight result of applying  Hoeffding's inequality with easy manipulation: with probability $1-\delta$
\begin{equation}
\label{eq:step1}
\Big|\frac{1}{m} \sum_{j=1}^m x_j- MMD(S_k(\G),S_k(\G'))^2 \Big| \leq\\ \frac{4 \sqrt{\log (2/\delta)}}{\sqrt{m}}
\end{equation}

\textbf{Step 2: finite $s$ and $m$.} For any \emph{fixed} set of random features $\{\mathbf{w}_j\}_{1,\ldots,m}$ and based on our previous assumptions we have: i/ $\varphi$ is in a ball of radius $M=\frac{\sqrt{m}}{\sqrt{m}}=1$, ii/ $ \mathbb{E}_{F \sim S_k(\G)}~ \varphi(F)= \mathbb{E}\left(~\frac{1}{{s}} \sum_i \varphi(F_i)~\right)$. Therefore, we can directly apply a vector version of Hoeffding's inequality \cite[Lemma 4]{Rahimi2009} on the vectors $\frac{1}{{s}} \sum_i \varphi(F_i)$ to get that with probability $1-\delta$:
\begin{equation}
    \label{eq:fixed_w}
    \left\|\mathbb{E}_{F \sim S_k(\G)} \varphi(F)-~\frac{1}{s} \sum_i \varphi(F_i)~\right\|\leq \frac{1+\sqrt{2\log\frac{1}{\delta}}}{\sqrt{{s}}}
\end{equation}
\vfill\pagebreak
Defining $J_{exp}(\G,\G')=\| \mathbb{E}_{F \sim S_k(\G)} \varphi(F) - \mathbb{E}_{F' \sim S_k(\G')} \varphi(F')\|$ and $J_{avg}(\G,\G')=\| \frac{1}{{s}} \sum_i \varphi(F_i) - \frac{1}{{s}} \sum_i \varphi(F'_i)\|$, then using triangular inequality followed by a union bound based on \eqref{eq:fixed_w}, we have the following with probability $1-2\delta$,
\begin{align*}
    \Big | J_{exp}(\G,\G') - J_{avg}(\G,\G')\Big | \leq \frac{2}{\sqrt{{s}}}\left(1+\sqrt{2\log\frac{1}{\delta}}\right)
\end{align*}

On the other hand, $ J_{exp}(\G,\G') + J_{avg}(\G,\G')\leq 4$, so with same probability:
\begin{equation}\label{eq:step2}
    \Big | J_{exp}(\G,\G')^2 - J_{avg}(\G,\G')^2 \Big | \leq \frac{8}{\sqrt{{s}}}\left(1+\sqrt{2\log\frac{1}{\delta}}\right)
\end{equation}
Since it is valid for any fixed set of random features, it is also valid with \emph{joint} probability on random features and samples, by the law of total probability.

Finally, combining \eqref{eq:step1}, \eqref{eq:step2} with a union bound and a triangular inequality, we have with probability $1-3\delta$,
\begin{align*}
\Big|\|\hat{\mathbf{f}}_{\G,S_k,\varphi} - \hat{\mathbf{f}}_{\G',S_k,\varphi} \|_2^2 - MMD(\G,\G')^2 \Big| \leq \\\frac{4 \sqrt{\log (2/\delta)}}{\sqrt{m}} + \frac{8}{\sqrt{{s}}}\left(1+\sqrt{2\log\frac{1}{\delta}}\right)
\end{align*}

which concludes the proof by taking $\delta$ as $\delta/3$.

\end{proof}
\end{appendices}